\documentclass{esannV2}
\usepackage[dvips]{graphicx}
\usepackage[latin1]{inputenc}
\usepackage{amssymb,amsmath,array,amsthm}
\usepackage{xcolor}

\newtheorem{theorem}{Theorem}
\newtheorem{corollary}{Corollary}

\newcommand{\quotes}[1]{``#1''}
\newcommand{\ie}{i.e. }

\begin{document}
\title{Theoretically Expressive and Edge-aware Graph Learning}

\author{Federico Errica$^1$, Davide Bacciu$^1$, Alessio Micheli$^1$
%
%
\vspace{.3cm}\\
%
$^1$ University of Pisa - Department of Computer Science \\
Largo Bruno Pontecorvo, 3, 56127, Pisa - Italy
}

\maketitle

\begin{abstract}
We propose a new Graph Neural Network that combines recent advancements in the field. We give theoretical contributions by proving that the model is strictly more general than the Graph Isomorphism Network and the Gated Graph Neural Network, as it can approximate the same functions and deal with arbitrary edge values. Then, we show how a single node information can flow through the graph unchanged.
\end{abstract}
%
%
\section{Introduction}

Graph Neural Networks (GNNs) \cite{nn4g,gnn-scarselli} have gained popularity as an efficient tool to process graph-structured data. The core idea underlying these models is the iterative aggregation of neighboring information to produce node representations. GNNs usually serve to solve node classification and graph classification tasks \cite{relational-inductive-biases-battaglia}.
In recent years, researchers have proposed many architectures that mainly differ in the way neighborhood aggregation (also known as graph convolution) is performed. Some GNNs are proven to be able to discriminate the same graphs as the Weisfeiler-Lehman (WL) test of graph isomorphism \cite{gin, wl-go-neural}, while others focus on modeling edge labels \cite{cgmm, ecc} and the recurrence in node representations \cite{ggs-nn}. 
In this work, we put these building blocks together to formalize a new family of GNNs that can handle arbitrary edges as well as the history of nodes and edges representations across layers. Our contributions are theorical: first, we show that the proposed network is strictly more expressive than the models it borrows from; then, we give further insights about contextual information spreading that add to those of \cite{nn4g}.
\section{Related Works}
Two are the GNNs that inspired this work. The Graph Isomorphism Network (GIN) \cite{gin} is capable of discriminating the same structures as the 1-dim WL-test of graph isomorphism, and its architecture is fairly simple and efficient. Instead, the Gated Graph Neural Network (GG-NN) \cite{ggs-nn} is designed to take into account the history of node representations across the layers of the architecture, whereas the aggregation function is not backed up with theoretical results. While there are no formal guarantees about the expressiveness of GG-NN, the inductive bias imposed by the Gated Recurrent Unit (GRU) \cite{gru} incorporated in the graph convolution allows node representations to seamlessly flow across \textit{layers}. This results in a neighborhood aggregation scheme that combines heterogeneous local \quotes{views} of the graph.
Finally, we mention that very few models for graph-structured data incorporate edge information in the learning process \cite{cgmm, ecc}, which is probably due to the fact that there are no common benchmarking datasets that contain attributed edge information. Nonetheless, the architecture we are about to define provides a theoretically more general tool to learn from graphs, as it combines the inductive bias of popular GNNs in a sound way.
\section{Model}
We now introduce our model, called Gated-GIN. We start by giving some notations; then, we present the details of the model.
\paragraph{Notation}
A graph $\mathbf{g} = (\mathcal{V}_g,\mathcal{E}_g,\mathcal{X}_g, \mathcal{A}_g)$ is formally defined by a set of nodes $\mathcal{V}_g$ and by a set of edges $\mathcal{E}_g$ between two vertices. Each node $u$ is associated with a vector $x_u \in \mathcal{X}_g$. A directed edge ($u$,$v$) between nodes $u$ and $v$ is represented by a vector $a_{uv} \in \mathcal{A}_g$. The neighborhood of a node $u \in \mathcal{V}_g$ is defined as $\mathcal{N}(u) = \{ v \in \mathcal{V}_g | (v,u) \in \mathcal{E}_g\}$, that is the set of nodes associated to incoming edges. We denote an hidden representation of a node $v$ with $h_v$ and that of an edge $(u,v)$  with $h_{uv}$. Finally, we speak of \textit{context} when a node's response depends on the information flowing through the structure, and we refer to the term \quotes{expressive} to say that a network is capable of approximating a certain family of functions.
\subsection{Definition}
Here, we extend the convolution of GIN \cite{gin} to deal with arbitrary edge values. Moreover, we incorporate the information propagation mechanism of GG-NN \cite{ggs-nn} to exploit the history of a node hidden representations across layers, rather than at different time steps. 
\paragraph{Node convolution} We start by defining the operations on attributed nodes at each layer $k$:
\begin{align}
& h^0_v = \phi^0_V(x_v), \nonumber \\
& z_v^k = \sigma(\mathbf{W}^V_z [(1+\epsilon^k_V)h_v^{k-1}, \sum_{u \in \mathcal{N}(v)}h_u^{k-1} \odot h^{k-1}_{uv}]  + \mathbf{b}^V_z), \nonumber \\
& r_v^k = \sigma(\mathbf{W}^V_r [(1+\epsilon^k_V)h_v^{k-1}, \sum_{u \in \mathcal{N}(v)}h_u^{k-1} \odot h^{k-1}_{uv}]  + \mathbf{b}^V_r), & \nonumber \\
& \tilde{h}_v^k = \phi^k_V((1+\epsilon^k_V)h_v^{k-1} \odot  r_v^k + \sum_{u \in \mathcal{N}(v)}h_u^{k-1} \odot h^{k-1}_{uv}), \nonumber \\
& h_v^k = (1-z_v^k) \odot h_v^{k-1} + z_v^k \odot \tilde{h}_v^{k}, \nonumber
\end{align}where $h^k$ is the hidden state, $\epsilon_V^k \in \mathbb{R}$ represents a learnable parameter, square brackets denote concatenation and $\odot$ the Hadamard product, $\sigma$ is a gated activation function such as the sigmoid, $\phi$ is a multi layer perceptron (MLP), the symbol $\mathbf{W}$ denotes a linear weight matrix and $\mathbf{b}$ its associated bias. The definitions of $z$ and $r$ are taken from GG-NN \cite{ggs-nn}, which, in turn, was inspired by the gating functions of GRU \cite{gru}. Indeed, $z$ and $r$ represent the \textit{update} and \textit{reset} gate, respectively.

\paragraph{Edge convolution} Similarly, we define edge representations at layer $k$ using node and edge representations computed at layer $k-1$:
\begin{align}
& h^0_{uv} = \phi^0_E(a_{uv}), \nonumber \\
& z_{uv}^k = \sigma(\mathbf{W}_E^z([h_{uv}^{k-1}, h_u^{k-1}, h^{k-1}_v]) + \mathbf{b}^E_z), \nonumber \\
& r_{uv}^k = \sigma(\mathbf{W}_E^r([h_{uv}^{k-1}, h_u^{k-1}, h^{k-1}_v]) + \mathbf{b}^E_r), \nonumber \\
& \tilde{h}_{uv}^k = \phi^k_E([h_{uv}^{k-1} \odot r_{uv}^k, h_u^{k-1}, h^{k-1}_v]), \nonumber \\
& h_{uv}^k = (1-z_{uv}^k) \odot h_{uv}^{k-1}+ z_{uv}^k \odot \tilde{h}_{uv}^{k}. \nonumber
\end{align}For undirected graphs, we can sum the contributions of $ h_u^{k-1}$ and $h^{k-1}_v$ rather than concatenating them, so that $z_{uv}=z_{vu}$ and $r_{uv}=r_{vu}$. Usually, the recurrent architecture of GRU shares parameters across time steps to deal with sequences of variable length. In our case, a \quotes{time step} refers to one of the layers used to construct the architecture, hence the use of weight sharing is at the discretion of the user. In the rest of the paper, we assume a weight sharing technique, but the theoretical analysis of Section \ref{sec:theory} is easily extendible.

\section{Theoretical Analysis}
\label{sec:theory}
This Section is devoted to provide a theoretical analysis of the proposed model. We start by proving that the method is at least as expressive as GIN \cite{gin}, which implies it can discriminate the same structures as the 1-dim WL test.
\begin{theorem}
\label{th:approximation-GIN}
Given a graph $g$ and a node $v \in \mathcal{V}_g$, let $h^k_v= GIN^k(g) \in \mathbb{R}^d$ and $\hat{h}^k_v= Gated$-$GIN^k(g) \in \mathbb{R}^d$ be the outputs of the $k$-th graph convolution layer of GIN and Gated-GIN, respectively. Let us further assume that the multiset of neighboring states is countable. Then, for any choice of parameters $\theta_{GIN}$ of a GIN architecture with $K$ layers, there exists a choice of parameters $\theta_{Gated-GIN}$ of a Gated-GIN architecture with $K$ layers such that, for each $0 \leq k\leq K-1$ and $\epsilon>0$, $||h^k_v- \hat{h}^k_v||<\epsilon$.
\end{theorem}
\begin{proof}
We proceed by induction. The statement trivially holds for $k=0$; indeed, node representations can be generated using the same MLP. We now assume the statement holds for $k-1$, and we will prove that it holds for $k\leq K$ as well. First, we ignore the presence of edges by setting $h^{k-1}_{uv}=1$. This can be done by choosing the parameters of the MLP associated with $\phi^k_E$ to represent the constant function $\phi^k_E(x) = \mathbf{1}$. It follows that we have $\tilde{h}^k_{uv}= \mathbf{1} \ \ \forall (u,v) \in \mathcal{E}_g$. Secondly, we need to ignore previous node representations, that is $h^{k}_{v} = \tilde{h}^{k}_{v}$. \\ To obtain this, it is sufficient that $z_v^k = 1$ and $r_v^k = 1$; this holds in the limit when $\mathbf{W}^V_z \rightarrow\mathbf{0}, \mathbf{W}^V_r \rightarrow\mathbf{0}, \mathbf{b}^V_z\rightarrow\mathbf{+\infty}$ and $\mathbf{b}^V_r\rightarrow\mathbf{+\infty}$, resulting in 
\begin{align*}
\lim_{\substack{b^V_r,b^V_z\to+\infty \\ W^V_r,W^V_z\to\mathbf{0}}} \hat{h}_v^k = \phi^k_V((1+\epsilon_V)h_v^{k-1} + \sum_{u \in \mathcal{N}(v)}h_u^{k-1}) = h_v^k.
\end{align*}
\end{proof}
Note that this proof is nearly identical when using MLPs instead of linear functions for the update and reset gates, as there is just one more matrix to consider. Moreover,  we follow \cite{gin} and focus on the case where input node features belong to a countable set, which is not restrictive in practice. \\ 
The following Theorem is analogous to the previous one but for GG-NN. Before going on, we informally define a \textit{multiset} as the set that allows for multiple instances for each of its elements.
\begin{theorem}
\label{th:approximation-GG-NN}
Given a graph $g$ and a node $v \in \mathcal{V}_g$, let $h^k_v= GG$-$NN^k(g) \in \mathbb{R}^d$ and $\hat{h}^k_v= Gated$-$GIN^k(g) \in \mathbb{R}^d$ be the outputs of the $k$-th graph convolution layer of GG-NN and Gated-GIN, respectively. Let us further assume that the multiset of neighboring states is countable. Then, for any choice of parameters $\theta_{GG-NN}$ of a GG-NN architecture with $K$ layers, there exists a choice of parameters $\theta_{Gated-GIN}$ of a Gated-GIN architecture with $K$ layers such that, for each $0 \leq k\leq K-1$ and $\epsilon>0$, $||h^k_v- \hat{h}^k_v||<\epsilon$.
\end{theorem}
\begin{proof}
We again proceed by induction. For $k$=0, we recall that $h^0_v = [x_v, \mathbf{0}]$, which can be obtained by a linear mapping $Wx_v$ where $W$ is a block matrix made by the identity matrix and the null matrix. Therefore, it follows from the universal approximation theorem \cite{universal-approximation-theorem} that $\hat{h}^0_v=\phi^0_V(x_v)$ can approximate $h^0_v$. \\ If we assume that for each $0<k\leq K$ and $\epsilon>0$, $||h^{k-1}_v- \hat{h}^{k-1}_v||<\epsilon$ and we use the same argument as in Theorem \ref{th:approximation-GIN} to ignore edge labels, the inductive step follows from Lemma 5 of \cite{gin}, \ie Gated-GIN can approximate any function defined on multisets.
\end{proof}
In this work, we are not interested in studying the relation between GIN and GG-NN, as we have provided an architecture that is capable of approximating both. The next corollary, however, states that Gated-GIN is strictly more general than both GIN and GG-NN, as it can also handle edge attributes.
\begin{corollary}
\label{cor:strictly-more-general}
The class of functions of Gated-GIN is strictly larger than those of GIN and GG-NN. 
\end{corollary}
\begin{proof}
We will prove the statement for GIN, but the proof is identical for Gated-GIN. Let $\mathcal{F}_{GIN}$ and $\mathcal{F}_{Gated-GIN}$ the set of functions that GIN and Gated-GIN can approximate, respectively. It follows from Theorem \ref{th:approximation-GIN} that $\mathcal{F}_{GIN} \subseteq \mathcal{F}_{Gated-GIN}$. Recall that, for any given graph $g$, $f \in \mathcal{F}_{\theta_{GIN}} $ignores the contribution given by $\mathcal{A}_g$. Therefore, $\mathcal{F}_{GIN}$ corresponds to the set of functions such that $h^k_{uv}=1 \ \ \forall k, (u,v)\in \mathcal{E}_g $. We conclude by saying that we can trivially construct a function $g \in \mathcal{F}_{Gated-GIN}$ such that $h^k_{uv}=0 \ \ \forall k, (u,v)\in \mathcal{E}_g $, hence $\mathcal{F}_{GIN} \subset \mathcal{F}_{Gated-GIN}$.
\end{proof}

Despite these results about the ability of GNNs to discriminate certain structures, little is known about the requirements needed to effectively spread information across the graph. In the following, we study what is needed for GNNs to diffuse a single node information across the graph.

\subsection{On context spreading of a single node}
The formal analysis of the context provided in \cite{nn4g} characterizes how all nodes spread information across a graph. Indeed, using a deep GNN with $k$ layers corresponds to making two nodes at distance $k$ (indirectly) exchange their information. Here, we show that some GNNs can, in theory, spread a \textit{single} node representation $h^k_{v}$ across the graph without altering its value. Note that this result only applies to GNNs that compute a parametrized weighted sum of neighbors.


\begin{theorem}
\label{th:spreading-node-info}
Given a graph $g$ and a node $v \in \mathcal{V}_g$, assume we want to propagate an arbitrary $h_v^k \neq \mathbf{0}$ to node $u$ at distance $d$ such that $h_u^{k'} = h_v^k, \ k'>k$. Then there exists a permutation invariant function on a multi-set $X$ of the form $g(X)=\phi(\sum_{x \in X}f(x))$ that can be approximated by the neighborhood aggregation of GNNs such that $k'=k+d$. 
\end{theorem}
\begin{proof}
The proof relies on the fact that an aggregation function that can make $h_v^k$ seamlessly flow through the graph is $g(X)= \frac{1}{Z}\sum_{x \in X}\delta_{x,h_v^k}*x$, where $\delta_{x,h_v^k}$ is the Kronecker delta and  $Z = \sum_{x \in X}\delta_{x,h_v^k}$ is a normalization term. This function is capable of ignoring values different from $h_v^k$, and takes an average when more than one value equal to $h_v^k$ appears in the multiset. In summary, $g(X) = h_v^k$ if and only if $h_v^k \in X$, and 0 otherwise. By using this argument with the result of Theorem 2 in \cite{nn4g}, it follows that there exists a $k'=k+d$ that satisfies $h_u^{k'} = h_v^k$. However, $\delta_{x,h_v^k}$ is a discontinuous function; as such, we need to show that it can be approximated by a continuous function, which in turn can be approximated by a neural network for \cite{universal-approximation-theorem}. To see this, consider the continuous function $f_{n,h_v^k}(h) = n^{-||h- h_v^k||_2}$; it is easy to show that the family of functions $\{f_{n,h_v^k}(h)\}, \ \ n > 1$ is \textit{pointwise convergent} to $\delta_{h,h_v^k}$. We conclude by saying that $f_{n,h_v^k}(h)$ can be approximated by an MLP with learnable parameter $n$.
\end{proof}
From a practical point of view, it may be very difficult to approximate $\delta_{h_v^k}(x)$ without imposing a more explicit inductive bias on the aggregation function. If the task at hand requires to move a node's information far away in the graph, one possibility is therefore to use the function $f_{n,h_v^k}(h) = n^{-||h, h_v^k||_2}$ to approximate $\delta_{h_v^k}$. Indeed, Theorem \ref{th:spreading-node-info} assumes $h_v^k$ is fixed, but we can treat it as a learnable parameter as well.

\section{Conclusions}
We have proposed a new architecture for GNNs that combines the inductive bias of the theoretically expressive Graph Isomorphism Network and the recurrent mechanism of the Gated Graph Neural Network. We proved that the architecture does not lose expressivity with respect to both GNNs, which means one can now combine all the benefits together with no compromise. Moreover, we incorporate edge convolutions to deal with arbitrary edge attributes. As a result, the new network is strictly more expressive than those considered in this work. Finally, we give a sufficient requirement for GNNs to spread a single node representation across the graph, which is of practical importance in applicative contexts. Future works include the empirical application of such an architecture to new benchmarks where edge information is crucial to solve a task.


\begin{footnotesize}

\bibliographystyle{unsrt}
\bibliography{bibliography}

\end{footnotesize}


\end{document}